%% file: attn_label_arxiv.tex
\documentclass{article} 
\usepackage{iclr2021_arxiv,times}

\input{math_commands.tex}

\usepackage[utf8]{inputenc} 
\usepackage[T1]{fontenc}    
\usepackage{hyperref}       
\usepackage{url}            
\usepackage{booktabs}       
\usepackage{amsfonts}       
\usepackage{nicefrac}       
\usepackage{microtype}      
\usepackage{graphicx}
\usepackage{algorithm}
\usepackage{algorithmicx}
\usepackage[noend]{algpseudocode}

\usepackage{listings}
\usepackage{xcolor}
\usepackage{amsthm}

\definecolor{codegreen}{rgb}{0,0.6,0}
\definecolor{codegray}{rgb}{0.5,0.5,0.5}
\definecolor{codepurple}{rgb}{0.58,0,0.82}
\definecolor{backcolour}{rgb}{0.95,0.95,0.92}

\lstdefinestyle{mystyle}{
    backgroundcolor=\color{backcolour},   
    commentstyle=\color{codegreen},
    keywordstyle=\color{magenta},
    numberstyle=\tiny\color{codegray},
    stringstyle=\color{codepurple},
    basicstyle=\ttfamily\footnotesize,
    breakatwhitespace=false,         
    breaklines=true,                 
    captionpos=b,                    
    keepspaces=true,                 
    numbers=left,                    
    numbersep=5pt,                  
    showspaces=false,                
    showstringspaces=false,
    showtabs=false,                  
    tabsize=2
}

\newtheorem{theorem}{Theorem}

\title{Learning Image Labels On-the-fly \\for Training Robust Classification Models}

\author{Xiaosong Wang, Ziyue Xu, Dong Yang, Leo Tam, Holger Roth, Daguang Xu \\
Nvidia Corporation, USA \\
\texttt{\{xiaosongw,ziyuex,dongy,leot,hroth,daguangx\}@nvidia.com} 
}

\iclrfinalcopy 
\begin{document}

\maketitle

\begin{abstract}
Current deep learning paradigms largely benefit from the tremendous amount of annotated data. However, the quality of the annotations often varies among labelers. Multi-observer studies have been conducted to study these annotation variances (by labeling the same data for multiple times) and its effects on critical applications like medical image analysis. This process indeed adds extra burden to the already tedious annotation work that usually requires professional training and expertise in the specific domains. On the other hand, automated annotation methods based on NLP algorithms have recently shown promise as a reasonable alternative, relying on the existing diagnostic reports of those images that are widely available in the clinical system. Compared to human labelers, different algorithms provide labels with varying qualities that are even noisier. In this paper, we show how noisy annotations (e.g., from different algorithm-based labelers) can be utilized together and mutually benefit the learning of classification tasks. Specifically, the concept of attention-on-label is introduced to sample better label sets on-the-fly as the training data. A meta-training based label-sampling module is designed to attend the labels that benefit the model learning the most through additional back-propagation processes. We apply the attention-on-label scheme on the classification task of a synthetic noisy CIFAR-10 dataset to prove the concept, and then demonstrate superior results (3-5\% increase on average in multiple disease classification AUCs) on the chest x-ray images from a hospital-scale dataset (MIMIC-CXR) and hand-labeled dataset (OpenI) in comparison to regular training paradigms.
\end{abstract}

\section{Introduction}
Supervised deep learning methods, although proven to be effective on many tasks, rely heavily on the quality of the data and its corresponding annotations. Some tasks enjoy almost error-free annotation, such as handwritten numbers and simple natural images. However, for other applications including most medical image analysis tasks, the inherent ambiguity of the task leads to unavoidable noise and fuzziness within the annotations themselves, no matter how experienced the expert labelers are. Meanwhile, under a multi-labeler setting for quality control purpose, the significant intra- and inter-observer variability injects even more uncertainties into the resulting labels. Beyond the above challenges, for the specific task of chest X-ray image classification, due to the fact that most labels of the available large-scale open datasets are automatically mined by Natural Language Processing (NLP) algorithms, there will be yet another layer of error-prone operation on top of existing variability. Ideally, we would prefer multiple manually and reliably labelled close-to-truth annotations, while in reality, most of the data only have a single annotation from an algorithm with relatively low accuracy. 

Learning to learn from a variety of data (labels) falls within the scope of meta-learning, which is popular in many machine learning application, e.g., domain adaptation/generalization \citep{finn2017maml,dou2019domain} and few-shot learning \citep{snell2017prototypical,liu2019few}. Those previous meta-learner models (as illustrated in Fig.~\ref{fig:learning}(b)) often focus on learning the distribution of data (inputs of tasks) and specifying the update strategy of learner model parameters. Indeed, data from different sets (distributions) will contribute to the final learner model. On the contrary, we do not want the model to learn from erroneous labels (from less-experienced labelers) but learn only from ``true'' labels. We utilize the learner model parameters (via a meta-training process) to sample ``true'' labels for training a single learner model (as shown in Fig.~\ref{fig:learning}(c)). 

To address these challenges,
we proposed an attention-on-label strategy to benefit the training from multiple labels on the same subject. Instead of the resource-demanding process of asking several human annotators to label the same data, we choose to utilize annotations from different algorithm-based labelers, which only add little overhead beyond the single-labeler scenario. A meta-training scheme is adopted and integrated in our proposed label-sampling module to attend the labels that benefit the model learning the most through additional back-propagation processes. 

Our contributions in this work are three-fold: 1) We proposed a training framework to compute the image labels on-the-fly in the classification tasks. A meta-training based approach is introduced to attend and sample the labels from multiple annotators; 
2) Gradient flows towards the label are investigated and implemented. Indeed, the multiple sets of labels are inputs to the training framework. Learnable operations of the labels will require additional updates of label-related model parameters;
3) We not only prove the concept on CIFAR-10 but also perform experiments on two chest X-ray datasets with both image-only and image-text classification tasks. In all datasets, superior performance of the proposed method is demonstrated in the image classification tasks compared to baseline methods. 

\begin{figure}[t]
    \centering
	\includegraphics[width=0.7\columnwidth]{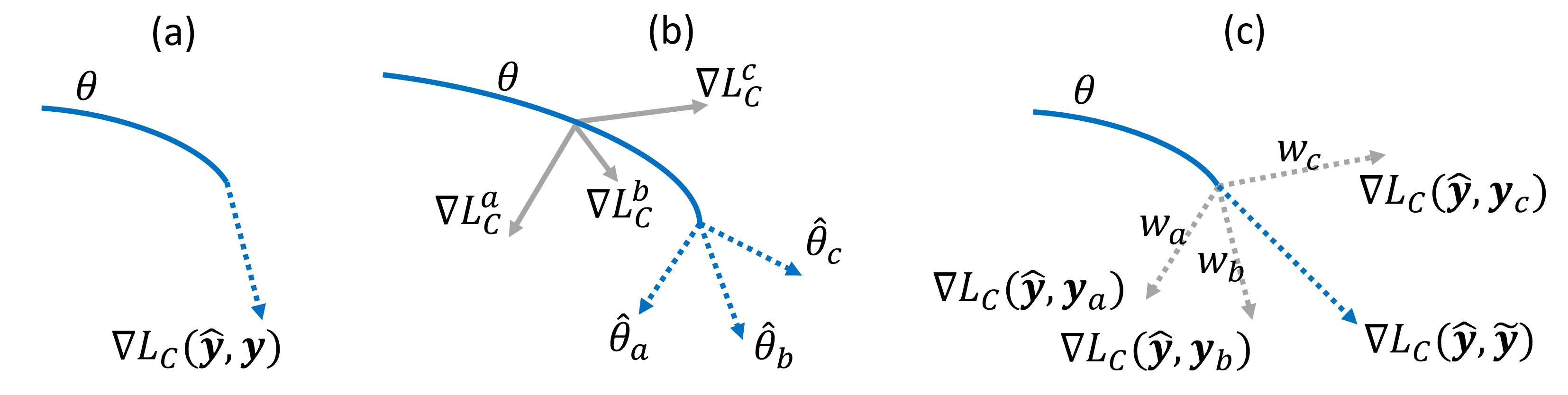}
	\caption{The diagram shows the difference of three learning paradigms in terms of how gradients are utilized for the training, i.e., (a) regular gradient based learning; (b) Meta-learning with multiple learning targets; (c) Proposed: learning the weights of each label ($\mathbf{y}_a,\mathbf{y}_b,\mathbf{y}_c$) in meta-training and then computing the weighted summation ($\tilde{\mathbf{y}}$) of labels for computing the final loss with prediction $\hat{\mathbf{y}}$.
	}
	\label{fig:learning}
\end{figure}

\section{Related Works}
\textbf{Meta learning: }Meta learning aims to learn a generalizable model by situating itself at a higher level than conventional learning. This can be achieved in several ways such as finding weights that can be easily adapted to other models~\citep{finn2017maml} or domains~\citep{li2018domaingen} during the training process. Meta learning results in models that can converge quickly with a few examples~\citep{ravi2017OptimizationAA}. They all share a similar meta-training process while the various goals of meta-training can divide them to different routes as examples shown in Fig.~\ref{fig:learning}. In this work, we target at weighting the importance of each label set based on its meta-training feedback and learning from the most effective labels.  

\textbf{Learning from noisy labels:} Learning from noisy labels \citep{natarajan2013learning,li2020dividemix,zhang2020distilling} has been a popular topic in deep learning due to its prevalence in many existing datasets with intra- and inter-observer variability, and the inherent uncertainties of both data and task themselves. For medical imaging applications with a high degree of ambiguity, this issue is even more significant. Recent works attempt to address this challenge via a consistency loss with a teacher model ~\citep{li2019learning}, loss weighting with 2nd order derivatives~\citep{zhang2020distilling}, and for medical image specifically, an online uncertainty sample mining strategy~\citep{xue2019ISBI}. Please note that they all focus on noise labels from a single annotator, while we attempt to design an attention mechanism to utilize labels from different sources together. 

\textbf{Multi-label classification in chest X-ray:} Because of its wide application and easy accessibility, chest X-ray is one of the major research areas in the field of medical image analysis. 
Among the pioneering works \citep{wang2017chestx,rajpurkar2017chexnet,yao2017learning, li2018thoracic,tang2018attention,hwang2019development,tang2020automated} in this area of deep learning, 
TieNet~\citep{wang2018tienet} first introduces an end-to-end trainable CNN-RNN architecture to extract distinctive text representations in addition to image features for improving label quality. More recently, a graph model was incorporated to enhance the learning accuracy~\citep{zhang2020radiology}. 

\textbf{Multi-observer studies}:
To ensure the annotation quality, especially for medical images where high expertise is required, it is common to have multiple set of labels on the same set of data \citep{rasch1999definition,cherian2005standardized}. In a sense, each annotation can be regarded as an estimation with uncertainty, good or bad, for the underlying ``true label''. 
Thus, algorithms taking this uncertainty factor into consideration are needed in order to make better use of such multi-observer data. \citet{kohl2018probabilistic} attempt to learn a distribution from a set of diverse but plausible segmentation from multiple graders. A recent work~\citep{tanno2019learning} proposed to model annotators by a confusion matrix which is jointly estimated during classification. 
In our work, instead of human annotators with different skill-levels, we employed several ``algorithmic labelers'' to generate multiple annotations from the same raw data. Comparing with their human counterparts, there exists less limitations and cost to increase the number of labelers, while the resulting labels can be more noisy. Hence, we choose a different strategy of attention model and meta-learning to benefit the model learning process.

\section{Learning from Data with Multiple Noisy Annotations}
The accuracy of NLP algorithm-based labelers has been studied \citep{peng2018negbio,irvin2019chexpert} and verified by a small set of hand-labeled data (based on associated report texts). The noises in annotations could be traced from many sources, e.g., algorithmic errors, incomplete information in reports, and misjudgement from the clinicians. All of these could elevate the uncertainty and impair the reliability of the publish data and associated ground-truth labels, in terms of their utilization in modern machine learning paradigms. 
``Who to believe?'' becomes a fundamental question to answer, which will ultimately have a significant impact on the performance of trained model.


MIMIC-CXR dataset \citep{johnson2019mimic} provides the labels sets from two independent algorithm based annotators with positive, negative, and uncertain cases. The availability of all these different sourced labels enables the observation of the uncertainty inherent to some data sample (with different values in multiple label sets). As shown in Table \ref{tab:uncertainty}, there are many uncertain cases in each kind of finding while the uncertainty / positive ratio may vary in diseases (ranged from $10\%$ to $110\%$).  

\begin{table}[t]
	\centering
	\resizebox{0.8\textwidth}{!}{
	\begin{tabular}{l|c|c|c|c|c|c}
		\hline
		\textbf{label sets} & atelectasis & cardiomegaly & consolidation & edema & pneumonia & pneumothorax \\
		\hline\hline
		negbio\_u           & 10986       & 11899        & 3348          & 13204 & 19029     & 1112        \\
		negbio\_p           & 47804       & 40509        & 11088         & 27911 & 16122     & 9885       \\
		\hline
		u/p ratio           & 0.229       & 0.293        & 0.301         & 0.473 & 1.18      & 0.112       \\
		\hline\hline
		chexpert\_u         & 10662       & 6235         & 4446          & 13817 & 18915     & 1177        \\
		chexpert\_p         & 47629       & 46373        & 11231         & 28339 & 16757     & 11046      \\
		\hline
		u/p ratio           & 0.223       & 0.134        & 0.395         & 0.487 & 1.128     & 0.106      \\
		\hline
	\end{tabular}}
	\caption{Uncertainties (\_u) and positives (\_p) of 6 sample disease findings from two labelers. }
	\label{tab:uncertainty}
\end{table}

Here, we try to tackle this problem by training a single multi-label classification model while considering all the available label sets. A novel attention-on-label training process is introduced. For each set of labels, we perform individual back-propagation as a form of meta-training and then compute the new image/image+text feature using the individual updated model. Based on the new features, we attended to the label set with more representative features (as a weight) and sample the weighted summary of labels for the final update of the model in each iteration. Fig.~\ref{fig:overview} illustrates the overall architecture and learning processes for both image and label model parameters. 


\begin{figure}[t]
	\centering
	\includegraphics[width=0.9\columnwidth]{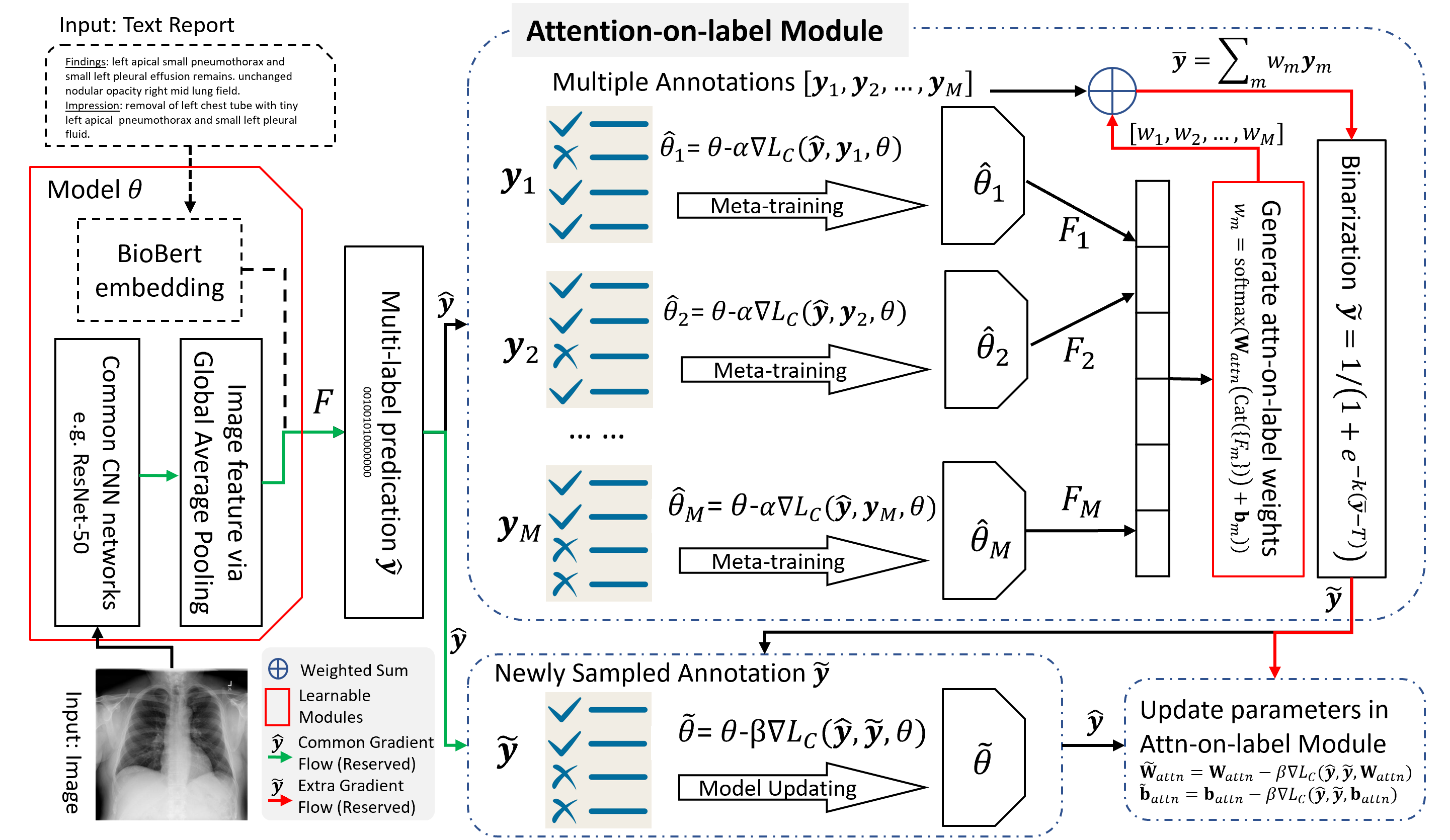}
	\caption{Overview of the proposed attention-on-label learning framework. 
	}
	\label{fig:overview}
\end{figure}

\subsection{Multi-label Classification Baseline}

We represent each image with $x$ and the labels of classes for each image as a binary label vector $\mathbf{y} = [y_{1},...,y_{n},...,y_{N}], y_{n}\in\{0,1\}$. 
$y_{n}=1$ indicates the presence of corresponding disease pattern or other findings in the image. In the multi-label disease classification setting, the presence of each finding is predicted separately by producing a likelihood after applying the sigmoid on each logit. For the experiments on CIFAR (as multi-class classification), we also use the one-hot binary vectors to represent the labels and predictions. Therefore, the proposed framework could be utilized for both multi-label and multi-class classification tasks.

Our proposed attention-on-label scheme could be applied to a large variety of pre-trained CNN architectures. Without loss of generality, we take the common ResNet-50 (from Conv1 to Res5c) as our backbone network. A global averaging pooling (GAP) layer was applied to transform the activation from convolutional layers into a one dimension image feature $F$. The reason for applying a GAP layer is the necessity to pass concatenated image and text features to a fully-connected layer for the final classification.   

We adopt the most common loss functions for the multi-label classification prediction $\hat{\mathbf{y}}$, i.e., binary cross entropy (BCE) loss:  
$\mathcal{L}_{C} (\hat{\mathbf{y}},\mathbf{y})=-1/N
\sum_{i=1}^N y_i \log \hat{y}_i+(1-y_i)\log(1-\hat{y}_i)$. 
Other more advanced losses could also be employed, while this type of improvement is out-of-scope in this paper. Here, we would like to demonstrate the feasibility and benefit of applying our proposed meta training process with attention-on-label over a vanilla model. Other critical issues, like the unbalanced numbers of pathology comparing with ``normal'' classes, are not considered here either to keep the evaluation simple and effective. 

\subsection{Attention on Labels}
The overall training procedure is illustrated in Algorithm \ref{alg:meta}. For each training iteration, we input each data entry $(x, Y)$ from the training set and $Y = \{\mathbf{y}_{1}, \dots, \mathbf{y}_{m}, \dots, \mathbf{y}_{M}\}$ are $M$ sets of image labels. 
During the meta-training, we compute the BCE loss ($\mathcal{L}_C(\hat{\mathbf{y}},\mathbf{y}_{m},\theta)$) between the prediction $\hat{\mathbf{y}}$ of current classification model $\theta$ and label set $\mathbf{y}_m$ and then perform the back-propagation to compute a new set of model parameters ($\hat{\theta}_m$),
\begin{equation}
\label{eqn:gd}
\hat{\theta}_m = {\theta}-\alpha \nabla_{\theta} \mathcal{L}_C (\hat{\mathbf{y}},\mathbf{y}_{m},{\theta}),
\end{equation}
\noindent $\alpha$ is the learning rate for this meta-training process.
We then compute a set of new features $\{F_{m}, m\in\{1,...,M\}\}$ via the inference of image $x$ using each meta-model $\hat{\theta}_m$ individually. $\{F_{m}\}$ could be either image features (i.e., output of the GAP) or one concatenated with text embedding (detailed in Section \ref{sec:image-text}). $\{F_{m}\}$ represent the feedback of model updates with each label set $\mathbf{y}_m$, i.e., the change that each $\mathbf{y}_m$ has brought to the model $\theta$. Other types of feedback from each noisy label could also be utilized here, e.g., the gradients $\{\nabla_{\theta} \mathcal{L}_C(\hat{\mathbf{y}},\mathbf{y}_{m},{\theta}), m\in\{1,...,M\}\}$. Here, we take $\{F_{m}\}$ as an example to compute the weight $w_{m}$ for each label set via a softmax based attention mechanism, 
\begin{equation}
\label{eqn:attn}
w_{m} = \mathrm{Softmax}(\mathbf{W}_{attn}(\mathrm{Cat}(\{F_{m}\})+\mathbf{b}_{attn}),
\end{equation}
\noindent where $\mathbf{W}_{attn}$ and $\mathbf{b}_{attn}$ are learnable parameters in our attention-on-label module. $\mathrm{Softmax}$ is the activation function. $\mathrm{Cat}$ represents the concatenation as a stack of all features. $w_{m}$ indicates the importance/correctness of label set $\mathbf{y}_m$ and is applied to compute the weighted average of all label sets for each data sample. Here, we employ a common softmax-based attention mechanism \citep{xu2015show,chen2017sca} as a sample case and many other more complex learning-based attention mechanisms can be adopted to compute the weights, e.g., self-attention \citep{vaswani2017attention}.

The values of label vectors after weighted average $\bar{\mathbf{y}}=\sum_m w_m \mathbf{y}_m$ are in the range of $[0,1]$, which is rather ambiguous for the model to learn. Binarization will be useful to cast the value close to either $0$ or $1$, but it is not differentiable and will disrupt the gradient flow. Therefore, we adopt the differentiable binarization function as first introduced in \citet{liao2019real},  
\begin{equation}
    \tilde{\mathbf{y}}=\frac{1}{1+e^{-k (\bar{\mathbf{y}} - T)}},
    \label{eq:db}
\end{equation}
\noindent where $k$ sets the sharpness of a $0$ to $1$ cliff. 
$T$ is a threshold to slightly adjust the value range.  
Finally, we update the model once more with the attended label and a globla learning rate $\beta$ for this iteration,
\begin{equation}
    \tilde{\theta} \leftarrow {\theta}-\beta\nabla \mathcal{L}_C (\hat{\mathbf{y}},\tilde{\mathbf{y}},{\theta}).
    \label{eq:final}
\end{equation}

\textbf{Gradient Flows Towards Labels:}
Extra gradient flows (highlighted in red in Fig.\ref{fig:overview}) are required for training our attention-on-label mechanism, specifically the parameters in Eq. \ref{eqn:attn}. Most of the current learning frameworks have gradient flows (highlighted in green in Fig.\ref{fig:overview}) with the images in the end since the labels are usually fixed or smoothed in advance \citep{muller2019does}. However, the gradients in our proposed framework not only flow to the images but also go through towards the labels since the final labels $\tilde{\mathbf{y}}$ are computed on-the-fly with learned weights/attentions. To our best knowledge, this concept of gradients towards labels is novel and has not been investigated and implemented before. Indeed, the inputs to the attention-on-label module are $M$ sets of labels and the computed features $\{F_{m}\}$, which are detached (without auto-computed gradients) and stacked during the meta-training. Therefore, additional parameter updating is required at the end of each iteration,
\begin{equation}
    \tilde{\mathbf{W}}_{attn} \leftarrow \mathbf{W}_{attn}-\beta\nabla \mathcal{L}_C (\hat{\mathbf{y}},\tilde{\mathbf{y}},\mathbf{W}_{attn}), \; \tilde{\mathbf{b}}_{attn} \leftarrow \mathbf{b}_{attn}-\beta\nabla \mathcal{L}_C (\hat{\mathbf{y}},\tilde{\mathbf{y}},\mathbf{b}_{attn}).
    \label{eq:final-label}
\end{equation}

\subsection{Image-text Embedding}
\label{sec:image-text}
Clinical textual material, e.g., clinical notes \citep{pelka2019branding} and radiology report \citep{wang2018tienet}, contains richer information. We include the text report as an input to the classification problem to see if our proposed learning process will still benefit the learning and further improve the classification accuracy. There are a variety of approaches to generate text embedding, e.g., Fisher vectors of word2vec \citep{klein2015associating}, bidirectional LSTMs \citep{wang2016image}, and the most recently developed BERT model \citep{devlin2018bert}. To keep the simplicity of our baseline model, we embed the text report to a 768 dimension real-valued vector using the uncased version of BioBert features \citep{lee2020biobert}, followed by two fully connected layers with 512 neurons each. 

\begin{algorithm}[t]
    \small
	\caption{Meta-training with the attention-on-label module}
	\label{alg:meta}
	\begin{algorithmic}[1]
		\State Randomly initialize ${\theta}$
		\For {each data entry $(x,Y)$}                	
    		\For {$m \in \{1:M\}$}
        		\State{Compute updated parameters with gradients: $\hat{\theta}_m = {\theta}-\alpha \nabla_{\theta} \mathcal{L}_c (\hat{\mathbf{y}},\mathbf{y}_m,{\theta})$  }   
        		\State{Compute new features $F_{m}$ using the newly updated $\hat{\theta}_m$}
    		\EndFor
    		\State Stack and concatenate the features: $\mathrm{Concat}(\{F_{m}\})$
    		\State Compute softmax attentions to generate the weight $w_{m}$ for each feature $F_{m}$ (Eq. \ref{eqn:attn})
    		\State Sample the new label $\bar{\mathbf{y}}$ for each data sample using the weight $w_{m}$
    		\State Perform differentiable binarization for each new label in $\bar{\mathbf{y}} \rightarrow \tilde{\mathbf{y}}$
    		\State Update the final image model $\tilde{\theta} \leftarrow {\theta}-\beta\nabla \mathcal{L}_C (\hat{\mathbf{y}},\tilde{\mathbf{y}},{\theta})$
    		\State Manually compute and back-propagate the gradients for the attention-on-label model ($\mathbf{W}_{attn},\mathbf{b}_{attn}$)
		\EndFor        
	\end{algorithmic}
\end{algorithm}


\section{Datasets}

\textbf{CIFAR-10}: We simulate 5 different types  of annotators (with different experience) in a similar manner as \citet{tanno2019learning} by injecting label noises into the training set of \textbf{CIFAR-10}, namely 1) hammer-spammer (HS), 2) structured-flips (SF), 3) ordered-confusion (OC), 4) Adversarial (AD), and 5) average (AVG) of 1-4. Each set of noisy labels are generated based on the defined confusion matrices for each type (as shown in Fig.~\ref{fig:noise_cm_main}). Whether each sample would have a noisy label is randomly selected while the overall noisy distribution should correspond to each confusion matrix, individually. Within all the noisy training data, we randomly select $20\%$ as the validation set.

\begin{figure}[t]
    \centering
	\includegraphics[width=0.9\columnwidth]{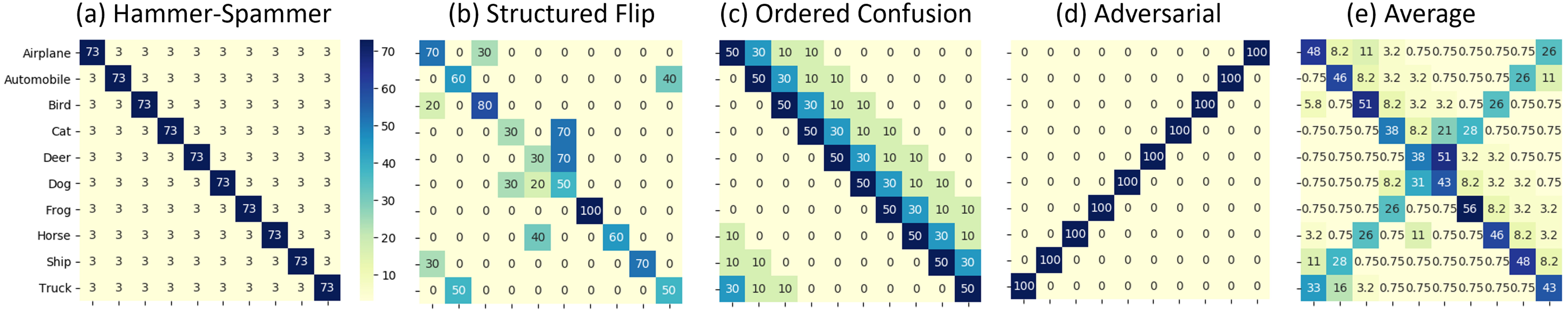}
	\caption{Confusion matrices of noisy labels from 5 different types of simulated annotators.
	}
	\label{fig:noise_cm_main}
\end{figure}

\textbf{MIMIC-CXR}: The MIMIC Chest X-ray \citep{johnson2019mimic} Database is a large publicly available dataset of chest radiographs with labels mined from image-associated text radiology reports using two different NLP based annotation tools, i.e., negbio \citep{peng2018negbio} and chexpert \citep{irvin2019chexpert}. The uncertain findings are marked as -1 in the original datasheet. Here, uncertainties are set to either 0 or 1 to form 4 different label sets, i.e., negbio\_u0, negbio\_u1, chexpert\_u0, and chexpert\_u1. The dataset contains 377,110 radiographs and labels from the 227,827 free-text radiology reports. Totally, 14 disease findings are listed ($N=14$). In our experiments, only the frontal view images are adopted, the number of which is equal to the number of reports.  We utilize the official data split for training, validation and testing. 

\textbf{MIMIC-CXR 1K hand-labeled Test}: Following the same labeling protocol proposed by \citet{demner2012design}, we randomly selected 1000 images and associated textual reports from the testing set of \textbf{MIMIC-CXR} and one of our staff (trained by a board-certified radiologist) hand-labeled the 1000 images by assigning the 14 labels manually to each image based on the reports, which will be released publicly. We believe it will also benefit the research in chest x-ray disease classification.

\textbf{OpenI}: OpenI \citep{demner2015preparing} is a public dataset of chest X-rays collected from multiple institutes by Indiana University.
In total, we fetch 3,851 unique radiology reports and 7,784 associated frontal/lateral images. To keep the consistency with \textbf{MIMIC-CXR} dataset, we use the same 14 categories of findings as mentioned above in the experiments. In our experiments, only 3,643 unique front view images and corresponding reports are evaluated.

\section{Experiments:}
The following methods in addition to the proposed method (\textbf{Ours}) are included in the comparison:

ResNet-50 (\textbf{R50}, \citealp{he2016deep}): We take the network based on ResNet-50 as a baseline. It adopts an ImageNet pre-trained ResNet-50 (from Conv1 to Res5c) as the backbone, followed by a GAP layer and a fully-connect layer for the final classification. Optionally, BioBert embedded text features will be concatenated with the output of GAP before the classification.     

\textbf{CM} \citep{tanno2019learning}: This method multiplies a confusion matrix with the probability that the model produces for each class. The basic assumption is that this confusion matrix can correct the missed labeled data and return the probability for the truth using the learned confusion matrices. We carefully implement it according to the code snapshot provided by the authors. 

\textbf{NG} \citep{zhang2020radiology}: This method utilizes the prior knowledge of the disease relations as a form of knowledge graph. By injecting such prior knowledge and employing a graph convolutional network, it learns the underlying info for the final classification and report generation task. It is worthy to note that the results we report in the experiments are produced by a model that is both trained and evaluated on the \textbf{OpenI} dataset. 

\textbf{TieNet} \citep{wang2018tienet}: It focuses more on how to learn the image and text embedding together using a CNN+RNN framework. 
Its LSTM based text embedding is relatively more complicated but also more representative through learning. We only adopt the text embedding from a pre-trained BioBERT model for our comparison, which is less customized.


\textbf{Evaluation Metric}: Receiver Operating Characteristic (ROC) curve is the standard metric to evaluate the performance of multi-label classification tasks. Here, Area Under the Curve (AUC) values are computed for all the experiments on \textbf{MIMIC-CXR} and \textbf{OpenI}. 
We compute the multi-class classification accuracy (using \textit{sklearn.metrics.accuracy\_score}) for all the experiments on \textbf{CIFAR-10}.

\subsection{Implementation Details:}
For pre-processing, we resize the image to
256$\times$256 (while keeping the size of 32$\times$32 for \textbf{CIFAR-10}) and normalize the image intensities to [0, 1]. No data augmentation is employed in our experiments. As mentioned above, we set the learning rate for the meta-training phase as $\alpha = 0.2$ and global learning rate as $\beta = 1e-4$. The best model for all hyper-parameters are determined via validation. We set $k=50$ and $T=0.5$ empirically in the differentiable binarization module. We use a single NVIDIA Titan-X Pascal for training each classification model with a uniform batch size $B = 32$ across all the experiments. Adam optimizer is utilized for training all the compared models.

\begin{table}[t]
	\small
	\centering
	\begin{tabular}{ l | c | c | c | c || c | c | c }
	    \hline
		 CIFAR-10 & HS & SF & OC & AD & AVG & \textbf{CM} & \textbf{Ours} \\
		\hline\hline
		 Noise-Level & 30\% & 40\% & 50\% & 100\% & 45\% & - & -  \\
		 Accuracy & 0.808 & 0.438 & 0.555 & 0.025 & 0.533 & \underline{0.593} & \textbf{0.677}\\
		\hline          
	\end{tabular}
	\caption{Averaged classification accuracy (with 5 noisy label sets) on \textbf{CIFAR-10} testing set (clean).}
	\label{tab:acc-cifar}
\end{table}

\subsection{Classification Results on CIFAR-10}
To prove the concept, we employ \textbf{CIFAR-10} with 5 types of added noises in labels to illustrate that the proposed attention-on-label scheme can be beneficial for the model training using multiple noise label sets. Table \ref{tab:acc-cifar} illustrates the multi-class classification accuracy on the \textbf{CIFAR-10} testing set. Noise-levels are computed in (1-Accuracy) using corresponding confusion matrices. In general, better quality labels and data lead to better trained models. Noise introduced by structured flips confuse the model training more than other types. Here, HS represents a more experienced set of annotators and models trained with it obtained a high accuracy. AVG represents the results of learning from a label set with simple average of the noise-level (defined by the confusion matrices) of all annotators. Our proposed method achieves over $12\%$ and $8\%$ performance improvements over AVG and the previous state-of-the-art \textbf{CM}. 

In addition, we also investigate how the label noise level and number of annotators will affect the model performance (detailed in Sec. \ref{sec:extra-cifar}). With noise level ranged from $10\%$ to $80\%$, \textbf{Ours} can constantly achieve better or similar results as the models trained with the relatively more accurate labels (e.g., AVG and HS). While experimenting with different number (2-5) of annotators, we observe a surge of the classification accuracy after including more than 2 annotators' labels. \textbf{Ours} often out-performs \textbf{CM} while the number of annotators increase, especially when labels with high noise levels are employed.

\subsection{Classification Results on Chest X-Ray Images}

\textbf{Classification Results Using Chest X-Ray Images}: Table \ref{tab:auc-image-text-short} shows the evaluation results for all the compared method using only the images as the input to the model. The left side of Table \ref{tab:auc-image-text-short} shows the AUCs of all the finding categories from \textbf{R50}, \textbf{CM} and \textbf{Ours}. The averaged AUC for \textbf{Ours} actually drops from the baseline. Considering the fact that the testing set of \textbf{MIMIC-CXR} dataset is also using the algorithm based labels, our predictions actually diverge from those noisy labels and lean to the underlying true labels. It is proved by the results illustrated in \textbf{OpenI} dataset section (right part of Table \ref{tab:auc-image-text-short}). \textbf{OpenI} dataset has the hand-labeled ground truth and our method is able to achieve over 4\% increase in the averaged AUC, which is also greater than what the \textbf{CM} method achieves. 
Although, both \textbf{NG} and \textbf{TieNet} partially utilized the report textual information in their image classification framework, \textbf{Ours} still is able to obtain equivalent or better results in most of the detailed disease categories.
As mentioned above, those disease categories with larger amount of uncertainties provide more information and therefore benefit more from the proposed meta-training process, e.g., Atelectasis and Consolidation. Please see the Appendix for more detailed results.

\begin{table}[t]
	\centering
	\resizebox{0.8\textwidth}{!}{
	\begin{tabular}{|l|c|c|c||c|c|c|c|c|}
	    \hline
		\textbf{AUC} & \multicolumn{3}{c||}{MIMIC-CXR Test (NLP)} & \multicolumn{5}{c|}{OpenI (hand-labeled)}                               \\ \hline
		Disease & \textbf{R50} & \textbf{CM} & \textbf{Ours} & \textbf{NG} & \textbf{TieNet} & \textbf{R50} & \textbf{CM} & \textbf{Ours} \\ \hline
		\multicolumn{9}{|c|}{Image only} \\
		\hline
		Atelectasis         & 0.821         & 0.832               & 0.826     & 0.833       & 0.774        & 0.781    & 0.81            & 0.826         \\ 
		Cardio.        & 0.825         & 0.852               & 0.879     & 0.913       & 0.847        & 0.859    & 0.881           & 0.879         \\ 
		Consol.       & 0.762         & 0.751               & 0.906     & -           & -            & 0.829    & 0.842           & 0.906         \\ 
		Edema               & 0.887         & 0.903               & 0.885     & 0.931       & 0.879        & 0.895    & 0.924           & 0.885         \\ 
		E-cardio     & 0.74          & 0.757               & 0.725     & -           & -            & 0.795    & 0.758           & 0.725         \\ 
		Fracture            & 0.722         & 0.771               & 0.632     & 0.671       & -            & 0.513    & 0.596           & 0.632         \\ 
		Lesion         & 0.765         & 0.744               & 0.643     & 0.643       & 0.658        & 0.585    & 0.58            & 0.643         \\ 
		Opacity        & 0.814         & 0.82                & 0.775     & 0.803       & -            & 0.742    & 0.738           & 0.775         \\ 
		No-finding          & 0.857         & 0.863               & 0.775     & -           & 0.747        & 0.754    & 0.739           & 0.775         \\ 
		Effusion    & 0.906         & 0.914               & 0.942     & 0.942       & 0.899        & 0.912    & 0.932           & 0.942         \\ 
		Pleural-o.       & 0.866         & 0.829               & 0.705     & -           & -            & 0.648    & 0.676           & 0.705         \\ 
		Pneumonia           & 0.809         & 0.809               & 0.871     & 0.863       & 0.731        & 0.781    & 0.823           & 0.871         \\ 
		Pneum-x        & 0.866         & 0.858               & 0.833     & 0.843       & 0.709        & 0.793    & 0.882           & 0.833         \\ 
		Devices     & 0.92          & 0.926               & 0.729     & 0.805       & -            & 0.628    & 0.655           & 0.729         \\ \hline
		Average             & 0.825         & 0.830               & 0.794     & -           & -            & {0.751}    & \underline{0.774}           & \textbf{0.794}         \\ \hline
		\multicolumn{9}{|c|}{Image \& Text} \\
        \hline
		Average             & 0.941         & 0.927               & 0.931     & - & -            & 0.809    & \underline{0.824}           & \textbf{0.835}         \\ \hline
	\end{tabular}}
	\caption{Classification AUCs for 14 findings in Chest X-Rays using the testing set of \textbf{MIMIC-CXR} (with NLP generated labels) and \textbf{OpenI} dataset (hand-labeled GT).  E-cardio: enlarged-cardiomediastinum; Pneum-x: pneumothorax. More details in Table \ref{tab:auc-image-text-full} of the Appendix.}
	\label{tab:auc-image-text-short}
\end{table}

\begin{table}[t]
\centering
\resizebox{0.8\textwidth}{!}{
\begin{tabular}{|l|c|c|c|c||c|c|c|}
\hline
\textbf{AUC} & \multicolumn{7}{c|}{MIMIC-CXR 1K hand-labeled Test} \\
\hline
 & negbio\_u1 & negbio\_u0 & chexpert\_u1 & chexpert\_u0 & \textbf{R50} & \textbf{CM} & \textbf{Ours} \\
\hline
\hline
Average &  \textbf{0.869} & 0.805 & \underline{0.868} & 0.809 & \underline{0.842}  & 0.837 & \textbf{0.868}  \\
\hline
\end{tabular}}
\caption{Averaged classification AUCs for 14 findings using the \textbf{MIMIC-CXR 1K hand-labeled Test} set. negbio\_u1, negbio\_u0, chexpert\_u1, and chexpert\_u0 are derived from the original NLP mined labels by setting the uncertainty to either 1 or 0. 
See all diseases' AUCs in the Appendix.}
\label{tab:mimic-1kgt-short}
\end{table}

\textbf{Classification Results Using Both Chest X-ray Images and Report Texts}: We observe similar improvements on the image-text classification task. Indeed, the text report contains more information about the disease diagnosis (maybe more than the image itself). We also observe the increase of the overall AUCs. In this case, our proposed meta-training with attention-on-label scheme also helps to boost the classification performance with a significant margin.  

As shown in Table \ref{tab:mimic-1kgt-short}, we further demonstrate the performance improvement from our proposed method on the \textbf{MIMIC-CXR 1K hand-labeled Test} data. Note that the absolute accuracy is higher than the ones reported on \textbf{OpenI} dataset. It may indicate the domain gap between \textbf{MIMIC-CXR} and \textbf{OpenI} datasets. Additionally, we evaluate the accuracy of all four NLP mined label sets. 
The choice of uncertainty makes the NLP labels have quite different accuracy (~6\% gap) on the \textbf{MIMIC-CXR 1K hand-labeled Test} set, 
while our proposed method actually achieve higher or similar accuracies in comparison to the NLP annotators while a model trained with NLP mined labels usually can not reach the same level of accuracy as the labels themselves, e.g., \textbf{R50} is trained with negbio\_u1. 

\section{Conclusion and Final Remarks}
In this paper, we introduced a novel learning framework (meta-training with the attention-on-label module) for handling data with multiple noisy label sets. The variability that the multiple label sets bring could in fact benefit the learning of a more accurate and robust model. We demonstrated the application of our proposed method on both multi-label and multi-class classification tasks and we believe it will be quite straight-forward to adapt it for other task, e.g., image segmentation. Indeed, the proposed method provides a viable way to handle the challenges of learning large-scale data with algorithm generated labels, where the label is a huge burden, especially for medical image analysis.  

The proposed label sampling module convert the hard labels ($0$ or $1$) to soft labels (a value between $0$ and $1$) and it also reduces the overfit towards erroneous labels, which is similar to the idea of label smoothing \citep{muller2019does}. Different to hard label smoothing, we assigned the new label on-the-fly (based on the network feedback) instead of instantly replacing 0 and 1 with $0+\epsilon$ and $1-\epsilon$. Such setting is extremely effective to handle the partial label errors in multi-label classification. Indeed, label smoothing is a form of loss-correction \citep{lukasik2020does} and we also prove the proposed label weighting is equivalent to directly performing the weighted average of the losses (generated from using different label sets) to form the final loss (see the proof in \ref{sec:prove}). 


\newpage
\bibliography{egbib}
\bibliographystyle{iclr2021_conference}

\appendix
\section{Appendix}
Due to the limited space in the main text, we want to discuss some important related issues here and meanwhile demonstrated additional experiments of the proposed attention-on-label module. It starts with a discussion about what the attention-on-label really does and follows by extensive experimental results on both CIFAR-10 and chest X-rays image datasets. 

\subsection{Attention-on-label Is A Form of Lost-correction}
\label{sec:prove}
Here, we try to demonstrate how the proposed attention-on-label module really works and discuss its connection to other prior arts. First, we present a theorem, indicating that attention weights applied on the labels can also be employed to re-weight the losses that are computed using label sets from different annotators. 

\begin{theorem}[Label-Sampling Formulation]
\label{thm:main_thm}
\textit{The loss computed with the model prediction $\hat{\mathbf{y}}$ and the sampled label $\tilde{\mathbf{y}}=\sum_{m=1}^{M}w_m \mathbf{y}_m$,
\begin{equation}
\mathcal{L}_{C} (\hat{\mathbf{y}},\tilde{\mathbf{y}})=\mathcal{L}_{C} (\hat{\mathbf{y}},\sum_{m=1}^M w_m \mathbf{y}_m)
\end{equation}
is equivalent to a weighted summation of losses computed with each set of label $\mathbf{y}_m$,
\begin{equation}
\mathcal{L}_{C} (\hat{\mathbf{y}},\tilde{\mathbf{y}})=\sum_{m=1}^M w_m \mathcal{L}_{C} (\hat{\mathbf{y}}, \mathbf{y}_m),
\end{equation}
where 
\begin{equation}
    \mathcal{L}_{C} (\hat{\mathbf{y}},\mathbf{y})=-\frac{1}{N}
\sum_{i=1}^N y_i \log \hat{y}_i+(1-y_i)\log(1-\hat{y}_i)
\end{equation} is the binary cross-entropy loss. 
$w_{m}$ is the attention weights as defined in the main text, $N$ represents the number of classes in the label, and $M$ is the total number of label sets (annotators).
}
\end{theorem}

Thus, the proposed attention-on-label mechanism can be transformed to a loss re-weighting module, where the weights are the same ones that  are generated from the back-propagation feedback using different label sets in the meta-training process. Indeed, it shares a similar insight with \citet{ren2018learning} while \citet{ren2018learning} re-weighted the losses based on the variation of input images. Similar to Theorem \ref{thm:main_thm}, we can also learn that re-weighting the losses is equivalent to the weighted sample of network predictions (maybe with different data samples as input). Accordingly, \textbf{the noise in data (on both images and labels) can be observed and corrected via studying the losses.} It also leads to a broader research topic of modeling the distribution of losses, e.g., recent work on learning with noisy label \citep{li2020dividemix} by modeling the losses with Gaussian mixture models. We believe our work probably could be explored further in that direction.

\begin{proof}[Proof of Theorem \ref{thm:main_thm}]
Recall that $\tilde{\mathbf{y}}=\sum_{m=1}^{M}w_m \mathbf{y}_m$ is the attention-weighted sampling of different label sets. 
$\mathbf{y}_m=[y_1^m,...,y_n^m,...,y_N^m\}, y_n^m\in\{0,1\}$ is one of the $M$ set labels.
\begin{align*}
    \mathcal{L}_{C} (\hat{\mathbf{y}},\tilde{\mathbf{y}})
    & = -\frac{1}{N} \sum_{i=1}^N [\tilde{y}_i\log(\hat{y}_i)+ (1-\tilde{y}_i)\log(1-\hat{y}_i)] \\
    & = -\frac{1}{N} \sum_{i=1}^N [\sum_{m=1}^{M}w_m y^m_i\log(\hat{y}_i)+ (1-\sum_{m=1}^{M}w_m y^m_i)\log(1-\hat{y}_i)] \\
    & = -\frac{1}{N} \sum_{i=1}^N \{\sum_{m=1}^{M}w_m [y^m_i\log(\hat{y}_i) - y^m_i\log(1-\hat{y}_i)]+ \log(1-\hat{y}_i)\} \\
    & = -\frac{1}{N} \sum_{i=1}^N \{\sum_{m=1}^{M}w_m [y^m_i\log(\hat{y}_i) + (1 - y^m_i)\log(1-\hat{y}_i) - \log(1-\hat{y}_i)] + \\ 
    & \quad\quad\log(1-\hat{y}_i)\} \\
    & = -\frac{1}{N} \sum_{i=1}^N \{\sum_{m=1}^{M}w_m [y^m_i\log(\hat{y}_i) + (1 - y^m_i)\log(1-\hat{y}_i)] + \\ 
    & \quad\quad (1-\sum_{m=1}^{M}w_m ) \log(1-\hat{y}_i)\}
\end{align*}
Since $w_m$ is computed using the $\mathrm{softmax}$ function and $\sum_{m=1}^{M}w_m=1$,
\begin{flalign*}
    \mathcal{L}_{C} (\hat{\mathbf{y}},\tilde{\mathbf{y}}) &
    =-\frac{1}{N} \sum_{i=1}^N \{\sum_{m=1}^{M}w_m [y^m_i\log(\hat{y}_i) + (1 - y^m_i)\log(1-\hat{y}_i)]\} \quad\quad\quad\quad\quad\\
    & = \sum_{m=1}^{M}w_m \{-\frac{1}{N} \sum_{i=1}^N [y^m_i\log(\hat{y}_i) + (1 - y^m_i)\log(1-\hat{y}_i)]\}  \\
    & = \sum_{m=1}^M w_m \mathcal{L}_{C} (\hat{\mathbf{y}}, \mathbf{y}_m)
\end{flalign*}
\end{proof}

\subsection{Additional Experiments on CIFAR-10}
\label{sec:extra-cifar}
\subsubsection{Details of Adding Noise to CIFAR-10 Training Data}
We simulate 5 different types  of annotators (with different experience) in a similar manner as \citet{tanno2019learning} by injecting a noise into the ground-truth label of the training set in \textbf{CIFAR-10}. 

\textbf{Hammer-Spammer (HS):}   
For each class, the annotation is correct with probability $p\in[0,1]$ and otherwise chooses
labels uniformly as defined in \citet{khetan2017learning} . 

\textbf{Structured-Flips (SF):}
Similar to the HS annotator, SF is correct with a probability $p$ and otherwise we flip the label of each class to another label (as easily confused ones), which is chosen as pre-defined pairs for each class. The pairs are defined as Airplane vs Bird, cat vs dog, dear vs cat, horse vs deer, ship vs airplane, and truck vs automobile.

\textbf{Ordered-Confusion (OC):}
This annotator is likely to confuse the target class with “neighbouring” classes as defined in the corresponding confusion matrix.

\textbf{Adversarial (AD):}
AD is an adversarial annotator, who has an high accuracy of grouping images from the same category while constantly giving the wrong labels. 

\textbf{Average (AVG):}
The AVG annotation is composed using a average of all 4 confusion matrices from previously defined 4 types of annotators.

The overall noise-level is defined as $1-p$. Each set of noisy labels are generated based on the defined confusion matrices for each type (as samples shown in Fig.\ref{fig:noise_cm}, where the noise-levels for HS, SF, OC, AD are 30\%, 40\%, 50\%, 100\% individually). Whether each data entry would have a noisy label is randomly selected while the overall noisy distribution should correspond to each confusion matrix, individually. 

\begin{figure}[t]
    \centering
	\includegraphics[width=1\columnwidth]{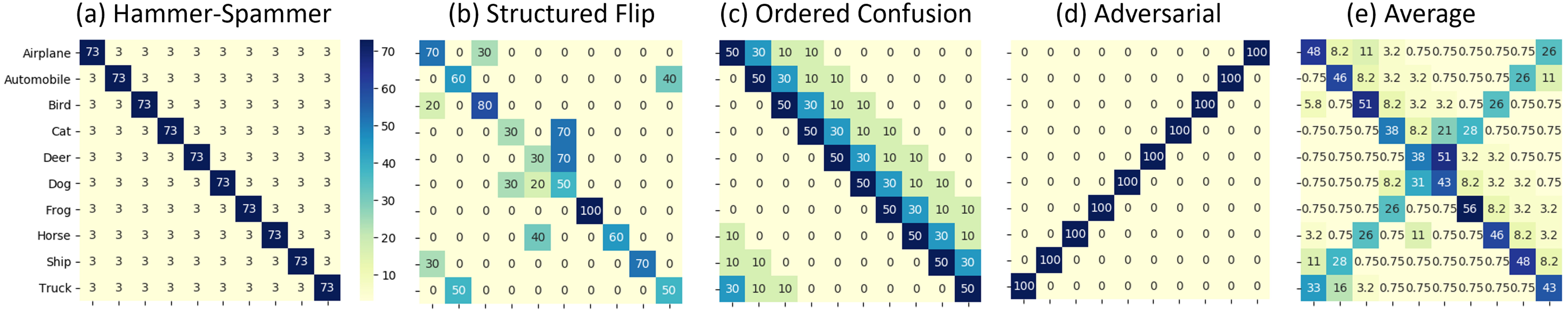}
	\caption{Confusion matrices of noisy labels from 5 different types of simulated annotators.
	}
	\label{fig:noise_cm}
\end{figure}


\subsubsection{How Does Different Noise Level of Labels Affect the Performance?} 
For this experiment, we want to see how the noise-level will affect the model training and the performance of our proposed method and \textbf{CM}. We include 4 sets of labels, i.e., HS, SF, OC and AVG. For each time, all 4 sets share the same noise-levels (ranged from 10\% to 80\%). We excuse the AD label set from this experiment since its noise level can not be adjusted. We compare the prediction accuracy (in multi-class classification tasks) of baseline models (\textbf{R50}) that are trained using individual label sets and our proposed method. The performance of \textbf{CM} is also evaluated. Both \textbf{CM} and \textbf{Ours} are using all 4 label sets for the training. The results are shown in Figure \ref{fig:noise_level}. The proposed method can constantly achieve better or similar results as the models trained with the relatively more accurate labels (e.g., AVG and HS). Furthermore, our method outperforms \textbf{CM} in most of the noise-levels and has even bigger margins for noise-level 40\% and 50\%. Additionally, we can also see that SF noise harms the performance the most since similar objects with erroneous labels can confuse the model more. 

\begin{figure}[t]
    \centering
	\includegraphics[width=1\columnwidth]{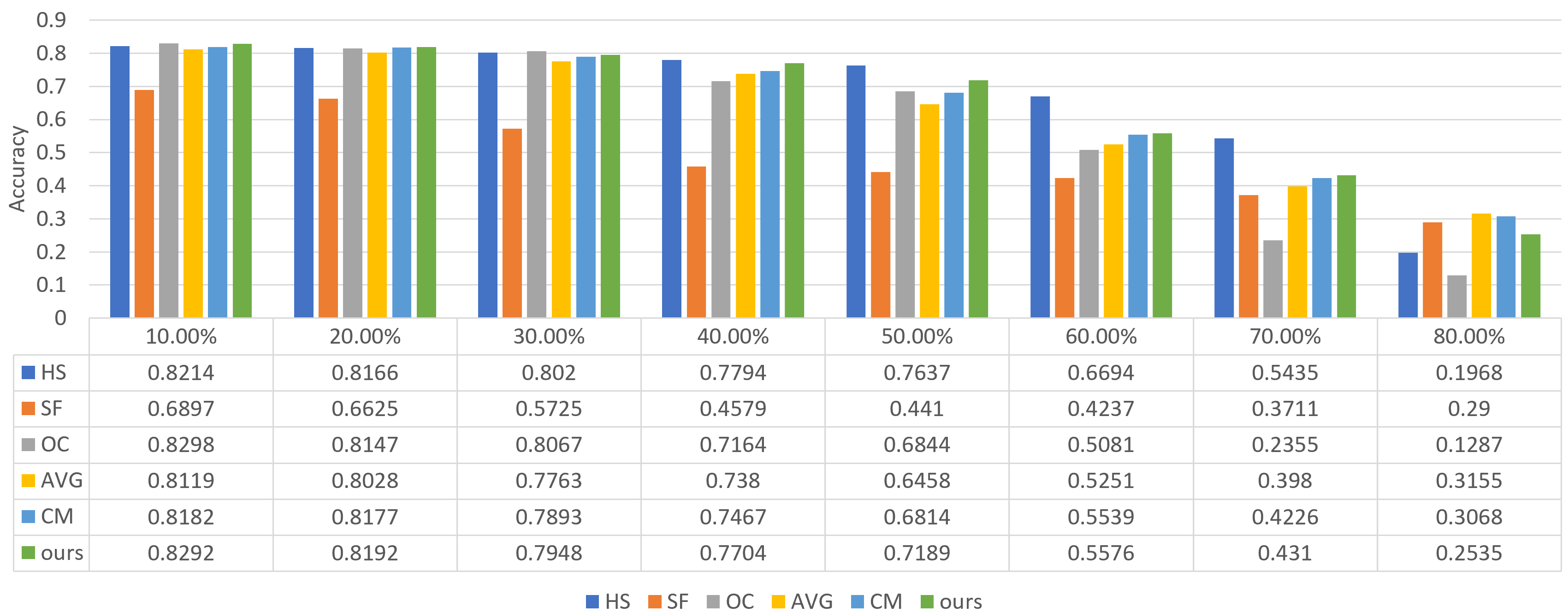}
	\caption{Classification accuracy using different label sets and methods over a range of noise-levels.
	}
	\label{fig:noise_level}
\end{figure}

\subsubsection{How Does Different Number of Available Annotation Sets Affect the Performance?}
In this experiment, we try to vary the number of available label sets (from 2, 3, 4 and 5 different types of annotators) used for training the proposed model (\textbf{Ours}) and \textbf{CM} at different noise level (at 10\%, 30\%, and 50\%). We start with training model using 2 quite different label sets, i.e., HS and AD. Then we add OC, SF, and AVG one at a time to see how the increasing number of the label sets could affect the final classification performance. HS represents a relatively experienced annotators and AD can be seen as a totally 'bad' annotator. We start with these two sets of labels. Then, label sets from other types of annotators are added. As shown in the Figure \ref{fig:annotators},  a surge of the accuracy can be observed after including more than 2 annotators. Considering that the adversarial annotator (with noise-level 100\%) is among the initial two, both \textbf{CM} and \textbf{Ours} can learn better immediately after a third one (as a confirmation) jumps in. We can also observe that our method performs much better when the noise level is relatively high (50\% in this case), while the CNN model itself may have already been capable of learning well from the labels with low noise levels. 

\begin{figure}[t]
    \centering
	\includegraphics[width=0.9\columnwidth]{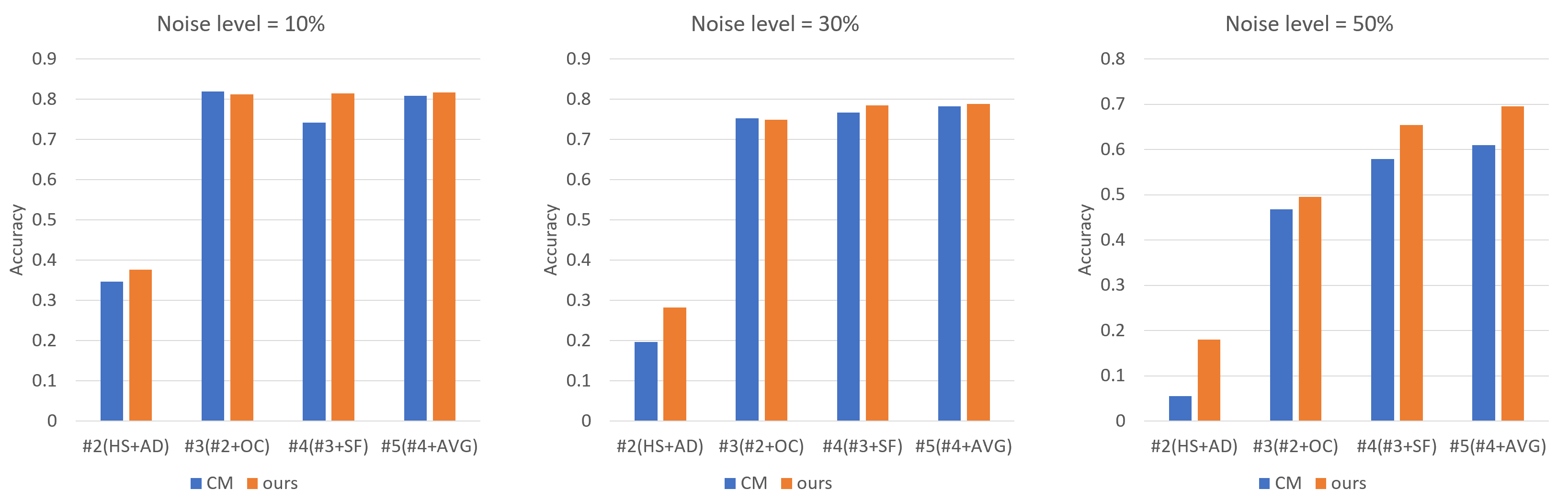}
	\caption{Classification accuracy using different number of label sets at noise-level 10\%, 30\%, and 50\%.
	}
	\label{fig:annotators}
\end{figure}

\subsection{Baseline Classification Results with Different Annotations} As shown in Table \ref{tab:uncertainty}, we observe a large difference in uncertainty among various algorithm generated label sets. Therefore, we first want to see how these different label sets will affect the model training. In Table \ref{tab:label-AUC}, we illustrate the averaged AUCs for four different label sets. $negbio\_u0$ is generated by setting all the uncertain cases to $0$ (as negative cases) while $negbio\_u1$ is produced by giving all the uncertain cases to $1$ (as positive cases). Similar process is applied to the $chexpert$ label sets as well. As shown in Table \ref{tab:label-AUC}, the testing performance of all four models are relatively on the same level for both \textbf{MIMIC-CXR} and \textbf{OpenI}. It indicates that the CNN based model is not so sensitive to the change of labels and can overcome the noise in the label set to a certain degree, but we should note that it does not necessarily improve the overall performance of the trained model. We believe a large amount of data with higher quality labels will benefit the training and make the trained model more accurate and robust.

\begin{table}[t]
	\small
	\centering
	\begin{tabular}{ l | c | c | c | c }
		\hline
		\textbf{Baseline Model} & negbio\_u1 & negbio\_u0 & chexpert\_u1 & chexpert\_u0 \\
		\hline\hline
		AVG AUC (MIMIC-CXR )    & 0.825      & 0.821      & 0.824        & 0.810        \\
		AVG AUC (OpenI)         & 0.751      & 0.756      & 0.755        & 0.752        \\
		\hline          
	\end{tabular}
	\caption{Averaged AUCs with the baseline R50 from four label sets.}
	\label{tab:label-AUC}
\end{table}

\subsection{Additional Experiments on Chest X-ray Images and Textual Reports}
In addition to the classification results using image only shown in the main text, we show the complete results for classifying the image and text together on the testing set of \textbf{MIMIC-CXR} (with NLP mined image labels) and \textbf{OpenI} (with hand-labeled labels). We observe similar results on the image-text classification task in comparison to image-only ones. The data from \textbf{OpenI} dataset are from a  different institute to the one of \textbf{MIMIC-CXR} data (our training data). Therefore, the domain gap shall be considered when we examine the performance. Table \ref{tab:auc-image-text-full} additionally shows the AUCs for each disease findings. Indeed, the text report contains more information about the disease diagnosis (maybe more than the image itself). We also observe increase of the overall AUCs. In this case, our proposed meta-training with attention-on-label scheme also helps to boost the classification performance with a significant amount. As pointed out before, TieNet achieves better classification results in some of the categories since they adopted a more complicated text embedding network (hard to implement with no open code released) and we believe better results could be obtained if our learning process was applied on the same model.

\begin{table}[t]
	\centering
	\resizebox{0.8\textwidth}{!}{
	\begin{tabular}{|l|c|c|c||c|c|c|c|c|}
	    \hline
		& \multicolumn{3}{c||}{MIMIC-CXR Test (NLP)} & \multicolumn{5}{c|}{OpenI (hand-labeled)}                               \\ \hline
		Disease & R50 & CM & Ours & NG & Tie & Rimages/50 & CM & Ours \\ \hline
		\multicolumn{9}{|c|}{Image only} \\
		\hline
		Atelectasis         & 0.821         & 0.832               & 0.826     & 0.833       & 0.774        & 0.781    & 0.81            & 0.826         \\ 
		Cardio.        & 0.825         & 0.852               & 0.879     & 0.913       & 0.847        & 0.859    & 0.881           & 0.879         \\ 
		Consol.       & 0.762         & 0.751               & 0.906     & -           & -            & 0.829    & 0.842           & 0.906         \\ 
		Edema               & 0.887         & 0.903               & 0.885     & 0.931       & 0.879        & 0.895    & 0.924           & 0.885         \\ 
		E-cardio     & 0.74          & 0.757               & 0.725     & -           & -            & 0.795    & 0.758           & 0.725         \\ 
		Fracture            & 0.722         & 0.771               & 0.632     & 0.671       & -            & 0.513    & 0.596           & 0.632         \\ 
		Lesion         & 0.765         & 0.744               & 0.643     & 0.643       & 0.658        & 0.585    & 0.58            & 0.643         \\ 
		Opacity        & 0.814         & 0.82                & 0.775     & 0.803       & -            & 0.742    & 0.738           & 0.775         \\ 
		No-finding          & 0.857         & 0.863               & 0.775     & -           & 0.747        & 0.754    & 0.739           & 0.775         \\ 
		Effusion    & 0.906         & 0.914               & 0.942     & 0.942       & 0.899        & 0.912    & 0.932           & 0.942         \\ 
		Pleural-o.       & 0.866         & 0.829               & 0.705     & -           & -            & 0.648    & 0.676           & 0.705         \\ 
		Pneumonia           & 0.809         & 0.809               & 0.871     & 0.863       & 0.731        & 0.781    & 0.823           & 0.871         \\ 
		Pneum-x        & 0.866         & 0.858               & 0.833     & 0.843       & 0.709        & 0.793    & 0.882           & 0.833         \\ 
		Devices     & 0.92          & 0.926               & 0.729     & 0.805       & -            & 0.628    & 0.655           & 0.729         \\ \hline
		Average             & 0.825         & 0.830               & 0.794     & -           & -            & {0.751}    & \underline{0.774}           & \textbf{0.794}         \\ \hline
		\multicolumn{9}{|c|}{Image \& Text} \\
        \hline
		Atelectasis         & 0.985         & 0.986               & 0.981     & - & 0.976        & 0.901    & 0.909           & 0.925         \\ 
		Cardio.        & 0.946         & 0.953               & 0.949     & - & 0.962        & 0.915    & 0.928           & 0.949         \\ 
		Consol.       & 0.911         & 0.913               & 0.904     & - & -            & 0.914    & 0.891           & 0.907         \\ 
		Edema               & 0.955         & 0.952               & 0.956     & - & 0.995        & 0.903    & 0.915           & 0.939         \\ 
		E-cardio     & 0.923         & 0.916               & 0.936     & - & -            & 0.581    & 0.714     images/      & 0.598         \\ 
		Fracture            & 0.935         & 0.766               & 0.876     & - & -            & 0.705    & 0.683           & 0.739         \\
		Lesion         & 0.865         & 0.885               & 0.814     & - & 0.96         & 0.615    & 0.607           & 0.649         \\ 
		Opacity        & 0.969         & 0.968               & 0.967     & - & -            & 0.849    & 0.854           & 0.877         \\ 
		No-finding          & 0.975         & 0.972               & 0.968     & - & 0.936        & 0.79     & 0.82            & 0.867         \\ 
		Effusion    & 0.974         & 0.973               & 0.974     & - & 0.977        & 0.944    & 0.948           & 0.943         \\
		Pleural-o.       & 0.92          & 0.877               & 0.892     & - & -            & 0.723    & 0.778           & 0.739         \\ 
		Pneumonia           & 0.927         & 0.931               & 0.933     & - & 0.994        & 0.812    & 0.834           & 0.889         \\ 
		Pneum-x        & 0.929         & 0.919               & 0.926     & - & 0.96         & 0.879    & 0.879           & 0.853         \\ 
		Devices     & 0.971         & 0.969               & 0.97      & - & -            & 0.796    & 0.787           & 0.821         \\ \hline
		Average             & 0.941         & 0.927               & 0.931     & - & -            & 0.809    & \underline{0.824}           & \textbf{0.835}         \\ \hline
	\end{tabular}}
	\caption{Classification AUCs for 14 findings in both Chest X-ray images and associated text reports using the testing set of \textbf{MIMIC-CXR} (with NLP generated labels) and \textbf{OpenI} dataset (hand-labeled GT). E-cardio: enlarged-cardiomediastinum; Pneum-x: pneumothorax}
	\label{tab:auc-image-text-full}
\end{table}

\subsection{Classification Results on MIMIC-CXR 1K Hand-labeled Test Data}
Following the same labeling protocol proposed in \citet{demner2012design}, we randomly selected 1000 images and associated textual reports from the testing set of \textbf{MIMIC-CXR} dataset. One of our staffs (trained by a board-certified radiologist) hand-labeled the image by assigning the 14 labels manually to each image based on the associated reports. Identifying disease mentions in a report is a more accessible process comparing to examine the chest x-ray images to recognize the disease pattern in images, which will require years of professional training. The 14 labels adopted are identical to the NLP mined labels, i.e., Atelectasis, Cardiomegaly, Consolidation, Edema, Enlarged-cardiomediastinum, Fracture, Lung-lesion, Lung-opacity, No-finding, Pleural-effusion, Pleural-other, Pneumonia, Pneumothorax, and Support-devices.

As shown in Table \ref{tab:mimic-1kgt}, we observe the similar superior performance from our proposed method on the \textbf{MIMIC-CXR 1K hand-labeled Test} data. Note that the absolute accuracy is higher than the ones reported on \textbf{OpenI} dataset. It may indicate the domain gap between \textbf{MIMIC-CXR} and \textbf{OpenI} dataset. 

Additionally, we evaluate the accuracy of NLP mined label sets, e.g., negbio\_u1 and negbio\_u0. negbio\_u1 and negbio\_u0 are two sets of image labels derived from the original NLP mined labels with uncertainty (labeled as -1). negbio\_u1 is produced by setting uncertainty to 1 and we set -1 to 0 for negbio\_u0. Two sets of NLP labels have quite different accuracy (~6\% gap) on the hand-labeled set. Similar findings can be observed between chexpert\_u1 and chexpert\_u0. It may be because the manual image labeling leans to the positive for those uncertain cases when our annotators will give 1 to those cases that the radiologist reported the observation of symptoms on the image (e.g., a density) while he/she is not sure what disease it is (e.g., it could be pneumonia, effusion, or consolidation). This aligns with the protocol used in \citet{demner2012design}. 
Our proposed method actually achieve a higher or similar accuracy as the NLP annotators while a model trained with NLP mined labels usually can not reach the same level of accuracy.   

\begin{table}[t]
\centering
\resizebox{\textwidth}{!}{
\begin{tabular}{|l|c|c|c||c|c|c|c|c|c|c|}
\hline
\textbf{AUC} & \multicolumn{3}{c||}{MIMIC-CXR Test (NLP-label)} & \multicolumn{7}{c|}{MIMIC-CXR 1k Test (hand-labeled GT)} \\
\hline
Disease & R50 & CM & Ours & R50 & negbio\_u1 & negbio\_u0 & chexpert\_u1 & chexpert\_u0 & CM & Ours \\
\hline
Atelectasis                 & 0.985          & 0.986         & 0.981         & 0.962 & 0.914            & 0.832     &0.908	&0.829  & 0.958 & 0.96  \\
Cardio.                     & 0.946          & 0.953         & 0.949         & 0.865 & 0.821            & 0.805     &0.822	&0.814  & 0.871 & 0.862 \\
Consol.                     & 0.911          & 0.913         & 0.904         & 0.802 & 0.892            & 0.772     &0.875	&0.777  & 0.796 & 0.861 \\
Edema                       & 0.955          & 0.952         & 0.956         & 0.877 & 0.955            & 0.901     &0.948	&0.903  & 0.880  & 0.912 \\
E-cardio                    & 0.923          & 0.916         & 0.936         & 0.744 & 0.847            & 0.759     &0.847	&0.759  & 0.745 & 0.769 \\
Fracture                    & 0.935          & 0.766         & 0.876         & 0.671 & 0.733            & 0.689     &0.747	&0.718  & 0.746 & 0.769 \\
Lesion                      & 0.865          & 0.885         & 0.814         & 0.76  & 0.777            & 0.727     &0.777	&0.719  & 0.739 & 0.759 \\
Opacity                     & 0.969          & 0.968         & 0.967         & 0.877 & 0.877            & 0.861     &0.871	&0.863  & 0.890  & 0.880  \\
no-finding                  & 0.975          & 0.972         & 0.968         & 0.909 & 0.845            & 0.845     &0.815	&0.815  & 0.906 & 0.909 \\
Effusion                    & 0.974          & 0.973         & 0.974         & 0.923 & 0.941            & 0.906     &0.940	&0.913  & 0.925 & 0.931 \\
pleural-o.                  & 0.92           & 0.877         & 0.892         & 0.803 & 0.906            & 0.825     &0.906	&0.825  & 0.749 & 0.844 \\
pneumonia                   & 0.927          & 0.931         & 0.933         & 0.872 & 0.957            & 0.685     &0.955	&0.694  & 0.878 & 0.902 \\
pneum-x                     & 0.929          & 0.919         & 0.926         & 0.852 & 0.874            & 0.831     &0.917	&0.860  & 0.788 & 0.889 \\
Devices                     & 0.971          & 0.969         & 0.97          & 0.872 & 0.838            & 0.837     &0.836	&0.837  & 0.860  & 0.885 \\
\hline
Average                     & 0.941          & 0.927         & 0.931         & 0.842 & \textbf{0.869}   & 0.805 &\textbf{0.868} &0.809 & 0.837 & \textbf{0.868}  \\
\hline
\end{tabular}}
\caption{Classification AUCs for 14 findings in Chest X-Ray image and text report using the testing set of \textbf{MIMIC-CXR} (with NLP generated labels) and \textbf{MIMIC-CXR 1K hand-labeled Test} set. negbio\_u1, negbio\_u0, chexpert\_u1, and chexpert\_u0 are four sets of image labels from the original NLP mined labels with uncertainty. negbio\_u1 is produced by setting uncertainty -1 to 1 and we set -1 to 0 for negbio\_u0. Similiar settings for chexpert\_u1 and chexpert\_u0.}
\label{tab:mimic-1kgt}
\end{table}

\newpage
\subsection{Pseudo-code of the Main Training Function}
\begin{lstlisting}[language=Python, caption=Train Function]
def train(cls_model, bert_model, attn-on-label, train_loader, args):
    
    for batch_idx, (images, labels_negbio_u1, labels_negbio_u0,
                    labels_chexpert_u1, labels_chexpert_u0,
                    reports) in enumerate(tqdm(train_loader)):
        cls_model.train()
        bert_model.eval()
        attn-on-label.train()
        optimizer.zero_grad()
        
        # computing model prediction
        text_embedding = bert_model(reports)
        preds = cls_model(images, text_feat)

        label_list = [labels_negbio_u1, labels_negbio_u0,
                      labels_chexpert_u1, labels_chexpert_u0]
        feat_all = []
        
        # meta-training for each set of label
        for targets_m in label_list:
            meta_loss = binary_cross_entropy_loss(preds, targets_m)

            # Meta-training for classification model
            grads = get_grad(meta_loss, cls_model.parameters())
            cls_weights = cls_model.parameters() - args.meta_lr * grads
            
            # Computing Image Features using Meta-trained model
            feat_tmp = cls_model(inputs, text_feat, weights=cls_weights,                 get_feat=True)
            feat_all.append(feat_tmp.detach())
            
        # attention-on-label module
        feat_all = stack(feat_all)
        attn_loss = softmax(attn-on-label.attn_fc(feat_all))
        # compute new label
        new_labels = mean(attn_loss * stack(label_list))
        # differentiable binarization
        new_labels = 1.0 / (1.0 + exp((new_labels - 0.1) * -50.0))
        
        class_loss = binary_cross_entropy_loss(preds, new_labels)

        # updating the attn-on-label parameters
        grads = get_grad(class_loss, attn-on-label.parameters())
        attn_weights = attn-on-label.parameter() - args.attn_lr * grad 
        attn-on-label.attn_fc.weight.copy(attn_weights['attn_fc.weight'])
        attn-on-label.attn_fc.bias.copy(attn_weights['attn_fc.bias'])
        
        # update the classification model
        class_loss.backward()
        optimizer.step()  
        
\end{lstlisting}

\end{document}

%% file: math_commands.tex
\usepackage{amsmath,amsfonts,bm}

\def\eqref#1{equation~\ref{#1}}

\def\1{\bm{1}}

\DeclareMathAlphabet{\mathsfit}{\encodingdefault}{\sfdefault}{m}{sl}
\SetMathAlphabet{\mathsfit}{bold}{\encodingdefault}{\sfdefault}{bx}{n}

